\documentclass{article}

\usepackage{arxiv}

\usepackage[utf8]{inputenc} 
\usepackage[T1]{fontenc}    
\usepackage{hyperref}       
\usepackage{url}            
\usepackage{booktabs}       
\usepackage{amsfonts}       
\usepackage{nicefrac}       
\usepackage{microtype}      
\usepackage{lipsum}		
\usepackage{graphicx}
\usepackage{doi}

\usepackage{amsmath}
\usepackage{amssymb}
\usepackage{amsthm}
\usepackage{xcolor}         
\usepackage{bbm}
\usepackage{bm}

\usepackage{wrapfig}
\usepackage{algcompatible} %
\usepackage{algorithm}     %
\usepackage{algorithmicx}  %
\usepackage{algpseudocode}
\usepackage{lipsum}
\usepackage{algpseudocode}
\usepackage[tableposition=top]{caption}

\title{Efficient and Robust Classification for Sparse Attacks}

\date{} 					

\author{{Mark Beliaev} \\
	University of California, Santa Barbara\\
	\texttt{mbeliaev@ucsb.edu} \\
	\And
	{Payam Delgosha} \\
	University of Illinois at Urbana-Champaign\\
	\texttt{delgosha@illinois.edu} \\
	\And
	{Hamed Hassani} \\
	University of Pennsylvania\\
	\texttt{hassani@seas.upenn.edu} \\
	\And
	{Ramtin Pedarsani} \\
	University of California, Santa Barbara\\
	\texttt{ramtin@ucsb.edu} \\
}

\date{}

\renewcommand{\headeright}{}
\renewcommand{\undertitle}{}

\hypersetup{
	pdftitle={A template for the arxiv style},
	pdfsubject={q-bio.NC, q-bio.QM},
	pdfauthor={David S.~Hippocampus, Elias D.~Striatum},
	pdfkeywords={First keyword, Second keyword, More},
}

\newcommand{\evwrt}[2]{\mathbb{E}_{#1} \left [ #2 \right ] }
\newcommand{\pr}[1]{\mathbb{P} \left ( #1 \right ) }

\newcommand{\norm}[1]{\left \Vert #1 \right \Vert}
\newcommand{\snorm}[1]{\Vert #1 \Vert}

\newcommand{\one}[1]{\mathbbm{1} \left [ #1 \right ]}

\newcommand{\reals}{\mathbb{R}}

\newtheorem{lem}{Lemma}
\newtheorem{thm}{Theorem}

\newtheorem{cor}{Corollary}

\newcommand{\tSigma}{\widetilde{\Sigma}}

\newcommand{\vtheta}{\vec{\theta}}
\newcommand{\vx}{\vec{x}}

\newcommand{\vZ}{\vec{Z}}
\newcommand{\vW}{\vec{W}}

\newcommand{\vz}{\vec{z}}
\newcommand{\va}{\vec{a}}
\newcommand{\vb}{\vec{b}}
\newcommand{\vs}{\vec{s}}
\newcommand{\vu}{\vec{u}}

\newcommand{\vw}{\vec{w}}
\newcommand{\tvw}{\widetilde{\vec{w}}}
\newcommand{\tw}{\widetilde{w}}
\newcommand{\vmu}{\vec{\mu}}

\newcommand{\vxp}{\vec{x}'}

\newcommand{\vnu}{\vec{\nu}}

\let\vec\bm

\newcommand{\mN}{\mathcal{N}}
\newcommand{\mX}{\mathcal{X}}

\newcommand{\mB}{\mathcal{B}}
\newcommand{\mC}{\mathcal{C}}
\newcommand{\mCkw}{\mathcal{C}^{(k)}_{\vw}}

\newcommand{\mD}{\mathcal{D}}

\DeclareMathOperator*{\argmin}{arg\,min}

\newcommand{\loss}{\mathcal{L}}
\newcommand{\optloss}{\mathcal{L}^\ast}

\newcommand{\sgn}{\text{sgn}}

\newcommand{\adv}{\mathsf{Adv}}

\newcommand{\ktrunc}{k^\text{Trunc}}
\newcommand{\kopt}{k^*}
\newcommand{\alphatrunc}{\alpha^\text{Trunc}}
\newcommand{\alphaopt}{\alpha^*}
\newcommand{\lambdace}{\lambda_{c(\varepsilon)}}

\algnewcommand\algorithmicinput{\textbf{Input:}}     %
\algnewcommand\INPUT{\item[\algorithmicinput]}       %
\algnewcommand\algorithmicoutput{\textbf{Output:}}   %
\algnewcommand\OUTPUT{\item[\algorithmicoutput]}     %
\algrenewcommand\algorithmicrequire{\textbf{Input:}} %
\algrenewcommand\algorithmicensure{\textbf{Output:}} %

\begin{document}
	\maketitle	
\begin{abstract}
	In the past two decades we have seen the popularity of neural networks increase in conjunction with their classification accuracy. Parallel to this, we have also witnessed how fragile the very same prediction models are: tiny perturbations to the inputs can cause misclassification errors throughout entire datasets. In this paper, we consider perturbations bounded by the $\ell_0$--norm, which have been shown as effective attacks in the domains of image-recognition, natural language processing, and malware-detection. To this end, we propose a novel defense method that consists of ``truncation" and ``adversarial training". We then theoretically study the Gaussian mixture setting and prove the asymptotic optimality of our proposed classifier. Motivated by the insights we obtain, we extend these components to neural network classifiers. We conduct numerical experiments in the domain of computer vision using the MNIST and CIFAR datasets, demonstrating significant improvement for the robust classification error of neural networks.
\end{abstract}
\section{Introduction} \label{sec:intro}
Today we see machine learning at the heart of many safety-critical applications, including image recognition, autonomous driving, and virtual assistance. This comes with little surprise, as we have seen deep neural networks gain tremendous popularity due to their success, showing near human performance in the image-recognition domain \cite{imagenet}, as well as successful application in natural language processing \cite{andor2016globally}, and playing games \cite{mnih2013playing,G0_article}. Instead, what is surprising is how fragile these neural networks are when subjected to adversarial attacks. \par
Adversarial attacks are methods that try to fool prediction models by adding small perturbations to their inputs. They were initially shown to be effective in causing classification errors throughout different machine learning models \cite{Biggio_2013,szegedy2014intriguing,goodfellow2015explaining}. Following this, a lot of effort has been put into generating increasingly more complex attack models that can utilize a small amount of semantic-preserving modifications, while still being able to fool a classifier \cite{37madry2019deep,13carlini2017evaluating,17croce2019sparse}. Typically, this is done by constraining the perturbations with an $\ell_p$--norm, where the most common settings use either $\ell_\infty$ \cite{52szegedy2014intriguing,32kurakin2017adversarial,13carlini2017evaluating,37madry2019deep,athalye2018obfuscated,marzi2018sparsity, add_4}, $\ell_2$ \cite{40moosavidezfooli2016deepfool,13carlini2017evaluating,48rony2019decoupling, add_1, add_2}, or $\ell_1$ \cite{14chen2018ead,39modas2019sparsefool}. As of now, the state-of-the-art empirical defense against adversarial attacks is iteratively retraining with adversarial examples \cite{37madry2019deep}. While adversarial retraining by itself can help improve robustness, we have seen a fundamental trade-off between robustness and clean accuracy, as well as a lack of generalization across different attacks \cite{tsipras2018robustness,su2018robustness,raghunathan2019adversarial,zhang2019theoretically,javanmard2020precise}. \par
In this paper we focus on a different setting, where adversarial perturbations are constrained using the $\ell_0$--norm. This setting has gained considerable~attention \cite{13carlini2017evaluating,42papernot2015limitations,39modas2019sparsefool,49schott2018adversarially,croce2020sparse,levine2019robustness} due to applications in object detection \cite{li2019adversarial,grosse2016adversarial} and NLP \cite{jin2019bert}. In these applications, robust guarantees against $\ell_0$--attacks are specifically important since there is an inherent limit on the number of input features that can be modified. In the previously described settings, the adversary was able to modify all of the elements of the input, while still satisfying the given constraint. Conversely, when using the $\ell_0$--norm the adversary is given a budget $k$, and directly constrained to perturbing at most $k$ coordinates within the input. In other words, the adversary is allowed to change the input within the $\ell_0$--ball of radius $k$, where $k$ is typically much smaller than the input dimension, and hence the name \emph{sparse attacks}. In addition, unlike $\ell_p$--balls ($p \geq 1$), the $\ell_0$--ball has a more complex geometry: it is non-convex, highly non-smooth, and unbounded. In combination with these properties, the $\ell_0$--ball's inherent discrete structure provides fundamental challenges that are absent in other adversarial settings studied in the literature, making most techniques from prior work non-applicable. Crucially, piece-wise linear classifier, e.g. neural networks with ReLU activations, were shown to fail in this setting \cite{shamir2019simple}, where recent work has demonstrated the success of $\ell_0$--attacks on images \cite{41narodytska2016simple,13carlini2017evaluating,42papernot2015limitations,49schott2018adversarially,17croce2019sparse}. Thus, our current architecture designs and learning procedures have to be rethought based on the unique geometry of the $\ell_0$--norm. We set out to accomplish this goal in this paper. 
\par
Two notable works have proposed defenses against the related but less powerful $(\ell_0+\ell_\infty)$--adversary: the \textit{Analysis by Synthesis} (ABS) model \cite{49schott2018adversarially} and randomized ablation \cite{levine2019robustness}. Here the adversary is also constrained by the number of coordinates it can perturb, but these perturbations can no longer be arbitrarily large due to the bound posed by the $\ell_\infty$--norm on the value that each coordinate can take. Although the proposed defenses show improved robustness guarantees when classifying the MNIST and CIFAR  datasets, we see these guarantees vanish as the $\ell_\infty$--bound is relaxed, while our method is able to generalize to both settings (more details are provided in Table~\ref{tab:median_l0} located in Section \ref{sec:experiments}). On top of this, we note that the aforementioned defenses rely on computationally expensive~solutions. \par
Building on our prior work \cite{delgosha2021robust}, we develop an algorithm that directly tackles the $\ell_0$ setting, and prove that in the Gaussian mixture setting we can achieve asymptotic optimality. 
Utilizing the state-of-the-art sparse attack of \texttt{sparse-rs} \cite{croce2020sparse} as well as the commonly used \texttt{Pointwise Attack} \cite{49schott2018adversarially}, we show that while adversarial training alone fails in robustifying against $\ell_0$--attacks, our method has strong performance both in terms of robustness and computational efficiency when tested on the MNIST  \cite{MNIST} and CIFAR \cite{CIFAR} datasets.\par
\section{Problem Setup} \label{sec:prob_setup}
We consider the general $M$--class classification problem, where given an input $\vx\in\reals^d$ and its  label $y\in\{1,\ldots,M\}$, we aim to construct a model that can accurately predict the  label given the input. We can think of the input and labels as coming from some distribution $(\vx,y)\in\mD$, with our classifier belonging to the family of functions $\mC:\reals^d\mapsto\{1,\ldots,M\}$. As a metric for the discrepancy between the  label and the classifier's prediction for a given input $\vx$, we use the $0-1$ loss $\ell(\mC;\vx,y)=\mathbbm{1}[\mC(\vx)\neq y]$. \par
Given this setup we can introduce an $\ell_0$--adversary, which perturbs the input $\vx$ within the $\ell_0$--ball of radius $k$: $\mB_0(\vx,k):=\{\vxp\in\reals^d:\norm{\vx-\vxp}_0\leq k\}$,
where we define $\norm{\vx}_0:=\sum_{i=1}^d\mathbbm{1}[x_i\neq0]$ for $\vx=(x_1,\ldots,x_d)$, and refer to $k$ as the $\textit{budget}$ of the adversary. This states that the adversary is allowed to arbitrarily modify at most $k$ coordinates of $\vx$ to obtain $\vxp$, feeding the new vector $\vxp$ to the classifier. Within this scope, the \textit{robust classification error} of a classifier $\mC$ is defined by:
\begin{equation} \label{eq:rob_class_error}
\loss_{\mD}(\mC, k)= \evwrt{(\vx, y) \sim \mD}{\max_{\vxp \in \mB_0(\vx, k)} \ell(\mC; \vxp, y)},
\end{equation}
where we aim to design classifiers with the minimum \textit{robust classification error}. To this end, we can define the \textit{optimal robust classification error} as the result of minimizing \eqref{eq:rob_class_error} over all possible classifiers:
\begin{equation} \label{eq:opt_loss_general}
\optloss_{\mD}(k) := \inf_{\mC} \loss_{\mD}(\mC, k).
\end{equation}
Due to the complex geometry of the $\ell_0$--ball, this poses a challenging problem. In fact, we have already seen how all conventional classifiers fail in this setting \cite{shamir2019simple}. In order to address this problem, our current architecture designs and learning procedures have thus to be \emph{rethought} based on the geometry of the perturbation set. To this end, we note that directly solving the optimization problem in \eqref{eq:rob_class_error} and finding the optimal robust error is not tractable. Instead, inspired by robust statistics \cite{huber2004robust}, we introduce truncation as the main building block of our classifier. We then aim to find the best robust classifier in the set of truncated classifiers. Such optimization can be analyzed in the Gaussian mixture scenario, and can be tackled by adversarial training in the general deep learning scenario. As shown in Section \ref{sec:main-results}, the theoretical study of the Gaussian mixture model allows us to establish the optimality of our method.\par
\section{The Proposed Algorithm}\label{sec:proposed_algo}
In this section we will go over the proposed algorithm, introducing how \textit{truncation} is defined, followed by an explanation of how it can be extended to \textit{fully connected} layers found within neural networks. We then describe the adversarial training component of our framework. As we will show in our theoretical and experimental results, coupling truncation with adversarial training is crucial to robustifying classifiers against $\ell_0$--attacks. We defer the explanation of applying truncation to convolutional networks to Section \ref{sec:experiments}, where we discuss our experiments using the CIFAR dataset. \par
\subsection{Truncation}
We define truncation as an operation that acts on two vectors by computing their truncated inner product. Given $\vw,\vx\in\reals^d$ and an integer $0\leq k\leq d/2$, we define the $k$--\textit{truncated inner product} of $\vw$ and $\vx$ as the summation of the element-wise product of $\vw$ and $\vx$ after removing the top and bottom $k$ elements, and denote it by $\langle \vw, \vx \rangle_k$. If we define $\vu:=\vw\odot\vx\in\reals^d$ as the element-wise product of $\vw$ and $\vx$, then letting $\vs=(s_1,\ldots s_n)=\textrm{sort}(\vu)$ be the result obtained after sorting $\vu$ in descending order, we can define
\begin{equation}\label{eq: truncation}
\langle\vw,\vx\rangle_k:=\sum_{i=k+1}^{d-k}s_i.
\end{equation}\par
Note that when $k=0$, the truncation operation in $\eqref{eq: truncation}$ reduces to the normal inner product denoted by $\langle\vw,\vx\rangle$. We can see that truncation is a natural method by which one can remove ``outliers'' found in the data after an adversary has modified some coordinates. Since an $\ell_0$--adversary with a budget of $k$ can modify at most $k$ of the input's coordinates by an arbitrary amount, we can expect the $k$--truncated inner product to be robust against these $\ell_0$ perturbations. In fact, we formalize this result in Section $\ref{sec:main-results}$ and show that truncation can be directly used to construct the optimally robust classifier in the setting of Gaussian mixture models attacked by an $\ell_0$--adversary. Until then, we will focus the discussion on how we use truncation to construct robust neural networks. \par
To test the usability of the proposed truncation operator, we must consider how it can be applied within typical neural network architectures to improve their robustness. Within the scope of our notation in Section \ref{sec:prob_setup}, we restrict the family of classifiers $\mC:\reals^d\mapsto\{1,\ldots,M\}$ to functions that can be represented by feed-forward neural networks composed of fully connected (FC) layers and non-linearities. \par
We denote a \textit{fully connected feed-forward neural network with $L$ layers} as a function $F(\vx;\vtheta)=y$ parameterized by $\vtheta$, which takes an input $\vx\in\reals^d$, and returns the predicted label $y\in\{1,\ldots,M\}$. This network can be viewed as a composite of $L$ functions, referred to as layers, with non-linearities applied between the layers:
\begin{equation}\label{eq:FC_net}
F(\vx;\vtheta) = \sigma_L(\vW_L\sigma_{L-1}(\vW_{L-1}\ldots\sigma_1(\vW_1\vx)\ldots)),
\end{equation}
where the parameters are $\vtheta = (\vW_1,\ldots,\vW_L)$ with $\vW_l\in\reals^{d_l\times d_{l-1}}$ and $d_0=d$, and the non-linearities are $\sigma_l:\reals^{d_l}\mapsto\reals^{d_l}$. In our work we use the well known ReLU \cite{relu} activation function for all of our non-linearities other than the one at the output layer $\sigma_L$, which is implemented as a softmax so that our function outputs a probability vector. Also note that we have left out denoting the bias terms added within the FC layers, as this can be taken care of by appending a constant coordinate to the input.\par
\subsection{Robust Fully Connected Networks}
We can naturally extend truncation to FC layers by defining this operation to act on a weight matrix $\vW$ as such: 
\begin{equation}\label{eq:matrix_trunc}
\langle \vW, \vx \rangle_k = \vu \textrm{, where } u_i= \langle \vW[i], \vx \rangle_k,
\end{equation}
using $\vW[i]$ to denote the $i$'th row of the weight matrix $\vW$. Note that \eqref{eq:matrix_trunc} returns a vector $\vu$, whose $i$'th entry $u_i$ is the result of applying our truncation operation shown in \eqref{eq: truncation} on the row $\vW[i]$ and vector $\vx$, where the biases are added after truncation is performed. To form our $k$--truncated fully connected network $F^{(k)}$, we replace the first FC layer $\vW_1\vx$ in \eqref{eq:FC_net} with its $k$--truncated version defined in \eqref{eq:matrix_trunc}.
\begin{equation}\label{eq:rob_FC_net}
F^{(k)}(\vx;\vtheta) = \sigma_L(\vW_L\sigma_{L-1}(\vW_{L-1}\ldots\sigma_1(\langle \vW_1, \vx \rangle_k)\ldots)).
\end{equation}
Note that with this formulation, $F^{(0)}=F$, since $\langle \vW, \vx \rangle_k = \vW\vx$ when $k=0$. Applying truncation on the first layer ensures that the effect of the adversary is compensated at the early stages of the network and does not propagate through the layers.\par
\subsection{Adversarial Training}
Although truncation on its own is expected to increase a classifier's robustness, we suggest going farther and coupling our framework with adversarial training as originally proposed by \cite{37madry2019deep}. In the Gaussian mixture setting considered in Section \ref{sec:main-results}, we prove that the asymptotically optimal classifier requires truncation as well as an optimization step for finding the best weights that resemble adversarial training. We hypothesize that extending these theoretical results to neural networks will help improve their robustness, and to this end we formalize the exact adversarial training algorithm we utilize in our work when testing our claim. \par
Our goal is to improve the robust guarantees of a FC network $F$ against an $\ell_0$--attack with budget $k$. We accomplish this by turning $F$ into its $k$--truncated counterpart $F^{(k)}$, and performing adversarial training on $F^{(k)}$ by iteratively appending adversarial examples to the training data. Of course adversarial training can be applied to any classifier $f$. Hence we express our training algorithm generally, by considering any $\ell_0$--adversary that attacks some classifier $f$ by using an $\ell_0$--budget of $k$ and a time budget of $t$. We define this attack as a function $g(\mX;f,k,t):\mX\mapsto\mX'$ where $\mX=\{\vx_1,\ldots,\vx_{|\mX|}\}$ is some set of unperturbed data samples, and $\mX'=\{\vxp_1,\ldots,\vxp_{|\mX'|}\}$ is a derived set of adversarial examples which are all misclassified by $f$. Note that we use $|\mX|$ to denote the cardinality of the set $\mX$. Using this attack, we train on the appended dataset $\mX\cup\mX'$, and every certain number of epochs we empty the adversarial set, and recalculate $\mX'=g(\mX;f,k,t)$. Hence the adversarial examples are chosen according to a procedure which is adaptive w.r.t. to our model $f$, and we use this procedure as a means of solving the minimax problem in \eqref{eq:opt_loss_general}. Note that we leave out the details of the \textit{training} framework used as this is problem specific and should be chosen accordingly.\par
\section{Theoretical Framework}
\label{sec:main-results}
In this section, within the setup of Section~\ref{sec:prob_setup}, we consider a Gaussian mixture setting and show that our
algorithm achieves  near optimal robust classification error, i.e.,
we show that the deviation from optimality is asymptotically vanishing. 
The key insight that we obtain from our theoretical analysis is that truncation and adversarial training are the two major components that enable provable robustness against $\ell_0$--attacks. 
More precisely, we consider the binary classification scenario where the distribution $\mD$ is as follows. We have $y \in
\{\pm 1\}$ with $\pr{y = +1} = \pr{y = -1} = 1/2$, and conditionally on $y$, we
have $\vx = y \vmu + \vz$ where $\vmu \in \reals^d$ and $\vz \sim \mN(0,\Sigma)$
is a Gaussian vector with zero mean and diagonal covariance matrix $\Sigma$. To
simplify the discussion, we assume that $\Sigma$ has strictly positive diagonal
entries $\sigma_1^2, \dots, \sigma_d^2$.\footnote{Although we make the diagonal assumption in this section, we discuss a more general setting in the Appendix}
It is easy to verify that in the
absence of the adversary, the optimal Bayes classifier is the linear classifier
$\sgn(\langle \vw, \vx \rangle)$ with $\vw = \Sigma^{-1} \mu$. The corresponding 
optimal standard error of this classifier is  $\bar{\Phi}(\snorm{\Sigma^{-1/2}
	\vmu}_2)$, where $\bar{\Phi}(.)$ denotes the complementary CDF of the
standard normal distribution. Therefore, in order to fix the baseline, 
without loss of generality we assume that $\snorm{\Sigma^{-1/2} \vmu}_2 = 1$ so that the optimal standard error is $\bar{\Phi}(1)$. 
Motivated by the fact that the optimal Bayes classifier in this setting is linear,
and referring our discussion in Section~\ref{sec:proposed_algo}, we consider
neural networks with a single layer.
More precisely, we consider the family of  $k$--truncated linear classifiers
$\mCkw  : \vxp \mapsto \sgn(\langle \vw, \vxp \rangle_k)$. 
Adopting our notation in~\eqref{eq:rob_class_error}, 
we denote the robust
classification error of a classifier $\mCkw$ in this family by $  \loss_{\vmu,
	\Sigma}(\mCkw, k)$.
Moreover, as in~\eqref{eq:opt_loss_general}, we denote the optimal robust
classification error by $  \optloss_{\vmu, \Sigma}(k)$.
To simplify the notation, when the problem parameters $\vmu$ and $\Sigma$ are
clear from the context, we may remove them from the above notations and simply
write $\loss(\mCkw, k)$ and $\optloss(k)$.

\subsection{Asymptotic Optimality of our Algorithm}
\label{sec:our-optimality}
To show that $k$--truncated linear classifiers are asymptotically optimal, we must first recall the following result from our prior work \cite{delgosha2021robust} which established a lower bound on the optimal robust classification by developing an attack strategy for the adversary and showing that no classifier can achieve better performance. 
\begin{thm}[Theorem 2 in \cite{delgosha2021robust}]
	\label{thm:lower-bound-diag}
	Assume that $\Sigma$ is diagonal and let $\vnu = \Sigma^{-1/2} \vmu$. 
	Then for any  $A \subseteq \{1,\dots,d\}$, we have
	\begin{equation*}
	\optloss\left(\snorm{\vnu_A}_1 \log d \right) \geq \bar{\Phi}(\snorm{\vnu_{A^c}}_2) - \frac{1}{\log d},
	\end{equation*}
	where $\vnu_A$ and $\vnu_{A^c}$ denote the coordinates of $\vnu$ in the sets
	$A$ and $A^c$, respectively.
\end{thm}
Recall from Section~\ref{sec:proposed_algo} that we use adversarial training in
order to obtain the model weights. This is a proxy for optimizing $\vw$ in the class of $k$--linear classifiers $\mCkw$. More precisely, let $\vw^*(k) \in \argmin_{\vw} \loss(\mCkw, k)$.
In the following, we show that the performance of $\mC^{(k)}_{\vw^*(k)}$ in the presence of an adversary with $\ell_0$ budget $k$ is comparable to the optimal robust classification error, with an asymptotically vanishing deviation.
In order to do this, 
given an error threshold $\bar{\Phi}(1) \leq \varepsilon \leq 1/2$, we define $\ktrunc(\varepsilon) := \max \{ k: \loss(\mC^{(k)}_{\vw^*(k)}, k) \leq \varepsilon\}$,
which is the maximum adversarial budget that the class of truncated linear
classifiers can tolerate to achieve a robust error of at most $\varepsilon$, with the truncation parameter chosen to be equal to adversary's budget. Here, $\varepsilon$ is chosen to range between the standard error $\bar{\Phi}(1)$ and the error corresponding to a random guess.
Moreover, let $\kopt(\varepsilon) := \max \{ k: \optloss(k) \leq \varepsilon\}$
be the maximum adversarial budget that an optimal   classifier can tolerate constrained
on having a robust error of at most $\varepsilon$.  Clearly $\kopt(\varepsilon)
\geq \ktrunc(\varepsilon)$.

As we will formally show below, $\ktrunc$ and $\kopt$ are close to each other
up to multiplicative factors that are sublinear in $d$. As a results, to have a first
order analysis and to focus on the behavior of the adversary's budget as a power
of the dimension $d$, we define $\alphatrunc(\varepsilon) := \log_d \ktrunc(\varepsilon)$ and $\alphaopt(\varepsilon) := \log_d \kopt(\varepsilon)$.
The following theorem shows that modulo some vanishing terms in $d$,
$\alphatrunc$  is close to $\alphaopt$. In other words, the class of linear
truncation classifiers are asymptotically optimal for the above mixture Gaussian setting.
Proof of
Theorem~\ref{thm:alpha-trunc-alpha-opt} is provided in Appendix~\ref{sec:app-thm-alpha-proof}.
\begin{thm}
	\label{thm:alpha-trunc-alpha-opt}
	Given $\bar{\Phi}(1) + 1/\log d + \sqrt{2/\log d} <
	\varepsilon < \frac{1}{2}$, there are constants $c_i = c_i(\varepsilon, d)$,
	$i \in \{1, 2\}$,  which do not depend on the parameters of the problem
	(i.e.\ $\vmu$ and $\Sigma$) such that $\lim_{d \rightarrow \infty}
	c_i(\varepsilon, d) = 0$ for $i \in \{1,2\}$ and
	\begin{equation*}
	\alphaopt(\varepsilon) \geq \alphatrunc(\varepsilon) \geq \alphaopt(\varepsilon - c_1)  - c_2.
	\end{equation*}
\end{thm}
Theorem~\ref{thm:alpha-trunc-alpha-opt} essentially says that up to asymptotically vanishing terms, the truncated classifier can tolerate as much adversarial budget as an optimal robust classifier. In order to prove this result, we use Theorem~\ref{thm:lower-bound-diag} which enables us to make sure that no other classifier can achieve better asymptotic performance, hence our algorithm is asymptotically optimal.
\section{Experiments} \label{sec:experiments}
To present our experimental results, we first discuss (i) how we chose and modified the $\ell_0$--attacks utilized in our experiments, and (ii) how under these modifications we saw the robust guarantees of prior work's previously proposed and well-studied $\ell_0$--defense method vanish. Following this in \ref{sec:Results on Image Datasets}, we show how our $k$--truncated FC networks performed on MNIST, and propose a heuristically motivated extension of truncation to $2$--dimensional convolution layers, testing it on the CIFAR dataset.\par
For our work, we mainly utilize \texttt{sparse-rs} \cite{croce2020sparse}, a sparse black-box $l_0$-attack framework. Given a pixel budget $k$, time budget $t$, input image $\vx$, and a prediction model $f$, this attack performs a random search where it tries to change a set of $k$ pixels in $\vx$ that cause the new adversarial image $\vx'$ to be misclassified by $f$. 
The creators of \texttt{sparse-rs} have shown their framework outperforms all previous black- and white-box attacks, and hence we use this attack within our adversarial training framework and after training to approximately measure the \textit{robust accuracy} of our classifier. 
We also utilize the \texttt{Pointwise Attack} \cite{49schott2018adversarially} to directly compare our results with other $\ell_0$-defense techniques \cite{levine2019robustness}. This attack tries to greedily minimize the $\ell_0$--norm by first adding salt-and-pepper noise, and then repeatedly resetting perturbed pixels while keeping the image misclassified. Since here we can not directly control the number of allowed perturbations $k$, we only use this attack to measure the \textit{median adversarial attack magnitude} as was done in prior work \cite{levine2019robustness}, denoting this value with $\rho$. \par
Before moving on, we point out that we normalize the coordinates of our inputs to be within some defined range $[-a,a]$. By design, the $\ell_0$--attacks mentioned also require the \emph{perturbed} coordinates to lie within some range $[-\beta a,\beta a]$, meaning they are indeed $(\ell_0+\ell_\infty)$ bounded. Formally, we define these attacks as being bounded by an $\ell_0$--norm of $k$ , and an $\ell_\infty$--norm of $\beta a$, where $\beta$ is a factor by which we scale the original domain $[-a,a]$. Since our goal is to develop a defense against a true $\ell_0$--attack, unless otherwise stated, we set $\beta=100$ as this effectively removes the $\ell_\infty$ constraint.\par
The two defenses we consider when comparing our proposed framework are: the \textit{Analysis by Synthesis} (ABS) model \cite{49schott2018adversarially} and randomized ablation \cite{levine2019robustness}. The ABS model relies on optimization-based inference by using variational auto-encoders that take $50$ steps of gradient descent, repeating this $1000$ times for each prediction. Defenses based on randomized ablation use thousands of ablated samples for each input to construct a set of images, following which the classifier performs a majority vote on this set to decide the best label for the original image. On the other hand, our method's computational complexity comes from the first $k$--truncated FC layer, where if the input array has dimension $d$, removing the top and bottom $k$ only adds $O(d)$ (when $k$ is constant) more operations per neuron, which is small compared to the overall complexity of deep neural networks. Hence our truncated network was still fast compared to the regular network.\par 
For the ABS model on MNIST, using \texttt{sparse-rs} with an $\ell_0$--budget of $12$ and a time budget of $10,000$ the robust accuracy decreases to $45\%$, which was significantly lower than the previously reported $78\%$. Additionally, the \texttt{Pointwise Attack} was used to calculate $\rho$ to be $22$ pixels. Note that both of these results were achieved for $\beta=1$, when testing these statistics for higher $\beta\in(1,100]$ we found that both robust guarantees vanish within the first hundred iterations i.e., the robust accuracy became $0\%$, and $\rho$ became $1$ pixel. For methods utilizing randomized ablation, robust guarantees were improved in relation to the ABS model: $\rho$ was reported to be $31$ pixels when $\beta=1$. Using code provided by the authors \cite{levine2019robustness}, we were able to confirm that $\beta=1$ was used in their experiments, unfortunately we could not test their robust accuracy with the stronger \texttt{sparse-rs} framework, nor could we increase $\beta$ to see if their defense would break similar to the ABS model. Due to these reasons, and the fact that truncation can act independently of ablation, we do not compare our results directly with theirs. \par 
\begin{figure}[!t]
	\centering
	\includegraphics[width=5in]{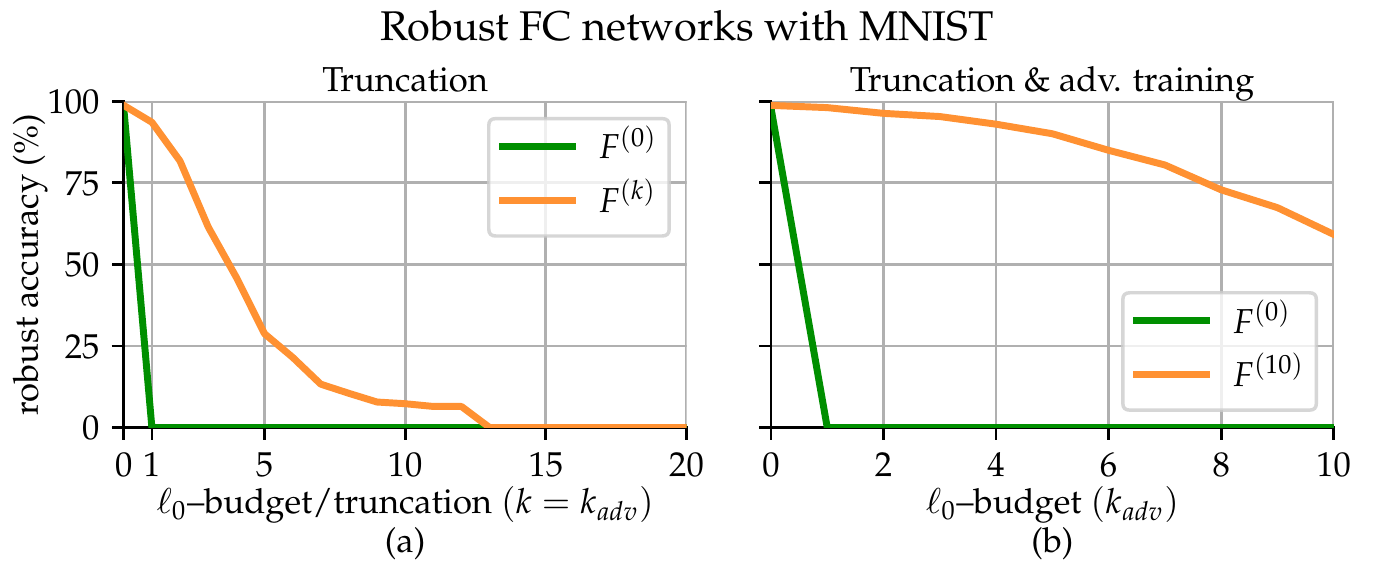}
	\caption{In (a) we show the robust accuracy of our $F^{(k)}$ (orange) and $F^{(0)}$ (green) without adversarial training, where$k=k_{adv}$ is shown on the x-axis. We see that $F^{(k)}$ outperforms $F^{(0)}$, but at $k\geq13$ the attack becomes too strong. In (b) we show the effect of adversarial training on $F^{(10)}$ (orange) and $F^{(0)}$ (green), varying the $\ell_0$--budget on the x-axis as $k$. We can see as compared to without adversarial training, $F^{(10)}$ has substantially improved.}
	\label{fig:MNIST}
\end{figure}
\begin{table}[!b]
	\centering
	\begin{tabular}{llllll}
		\toprule
		\multicolumn{3}{c}{Setup} &
		\multicolumn{3}{c}{Robust acc. \texttt{sparse-rs} (\%)}\\
		\cmidrule(r){1-3}\cmidrule(r){4-6}
		Architecture & Clean acc. (\%) & $\ell_0$--budget & $t=300$ & $t=1000$ & $t=5000$\\
		\midrule
		$F^{(0)}$ & $98.02$ & $3,5,8$ & $0.00$ & $0.00$ & $0.00$ \\
		$F^{(10)}$ & $98.73$ & $3$ & $95.51$ & $94.73$ & $92.97$ \\
		$F^{(10)}$ & $98.73$ & $5$ & $93.55$ & $89.65$ & $81.84$ \\
		$F^{(10)}$ & $98.73$ & $8$ & $85.94$ & $73.24$ & $58.79$ \\
		\midrule
		$\textrm{VGG}^{(0)}$ & $87.68$ & $3$ &  $64.45$ & $52.73$ & $39.65$ \\
		$\textrm{VGG}^{(0)}$ & $87.68$ & $8$ & $52.92$ & $40.23$ & $26.36$ \\
		$\textrm{VGG}^{(10)}$ & $87.27$ & $3$ & $77.73$ & $71.67$ & $67.77$  \\
		$\textrm{VGG}^{(10)}$ & $87.27$ & $8$ & $70.70$ & $61.33$ & $53.13$ \\
		\bottomrule
	\end{tabular}
	\caption{Adversarial training using \texttt{sparse-rs} on MNIST and CIFAR. The table above shows the final robust accuracy of $F^{(10)}$ and $F^{(0)}$ after adversarial training on MNIST, as well as $\textrm{VGG}^{(0)}$ $\textrm{VGG}^{(10)}$ on CIFAR. We give the clean accuracy (Acc. \%) of the classifiers along with the $\ell_0$--budget used to attack them. We then show the robust accuract (Rob. acc.) as we vary the adversary's time budget $t$. Note $F^{(0)}$ fails for any budget greater than zero.}
	\label{tab:rob_acc}
\end{table}
\subsection{Results on MNIST and CIFAR}\label{sec:Results on Image Datasets}
We begin by discussing our results when testing the proposed $k$--truncated FC network on the MNIST dataset. All networks $F^{(k)}$ were trained via stochastic gradient descent, and had the same architecture, consisting of $5$ FC layers with ReLU activations between them, where the first layer was replaced with the $k$--truncated matrix transformation from \eqref{eq:matrix_trunc}. \par
First, we look at the affect the truncation parameter $k$ and $\ell_0$--budget $k_{adv}$ have on the initial robust accuracy, without adversarial training. 
We can see the strength of the attack portrayed in Fig.~\ref{fig:MNIST}a, where the unprotected network $F^{(0)}$ fails for $k_{adv}\geq1$, and even $F^{(k)}$ becomes fully susceptible to $\ell_0$--attacks with budget $k_{adv}\geq13$. We set out to improve the robustness of the specific classifier $F^{(10)}$ via adversarial training, where we demonstrate this robustness by testing $F^{(10)}$ against $\ell_0$--attacks with budgets $k_{adv}\leq10$.\par
For adversarial training we used the \texttt{sparse-rs} attack with $\ell_0$--budget $k_{adv}=10$ and time budget $t=300$ queries. 
With this in mind, we believe our robust accuracy should be tested with an attack of similar time budget. However, we use a much larger time budget of $t=5000$ queries for the results displayed in Table~\ref{tab:rob_acc}, while in Fig.~\ref{fig:MNIST} we use $t=1000$ queries.\par
We can see from Fig.~\ref{fig:MNIST}b that adversarial training improves the robust accuracy of our $k$--truncated classifier, agreeing with our theory. When comparing to the initial results in Fig.~\ref{fig:MNIST}a, adversarial training shows no effect on the robust accuracy of the regular classifier $F^{(0)}$, while displaying substantial improvements when applied to $F^{(10)}$.\par
\begin{table}[!t]
	\centering
	\begin{tabular}{llll}
		\toprule
		\multicolumn{2}{c}{Setup} &
		\multicolumn{2}{c}{Median (pixels)} \\
		\cmidrule(r){1-2}\cmidrule(r){3-4}
		Architecture & Dataset  & $\beta=100$ & $\beta=1$\\
		\midrule
		$F^{(0)}$  & MNIST & $1$ & $13$\\
		$F^{(10)}$ & MNIST  & $\mathbf{17}$ & $\mathbf{21}$\\
		$\textrm{VGG}$  & CIFAR & $2$ & $3$  \\
		$\textrm{VGG}^{(10)}$ & CIFAR  & $11$ & $17$\\
		\bottomrule
	\end{tabular}
	\caption{$\rho$ using the \texttt{Pointwise Attack}. The table above shows the median adversarial attack magnitude denoted as $\rho$ for both our fully connected and convolution networks. Note that the ABS model as well as randomized ablation are not effective when $\beta=100$, while for $\beta=1$ the ABS model achieves an identical $\rho=21$.}
	\label{tab:median_l0}
\end{table}
We highlight these results in Table~\ref{tab:rob_acc}, showing that for lower budgets $k_{adv}$ we can maintain high robust accuracy even as the time budget $t$ increases. Also, there is no loss in classification accuracy from truncation as both $F^{(0)}$ and $F^{(10)}$ reach the same clean accuracy after adversarial training, which is slightly
lower than the base classifier's clean accuracy of $99.3\%$. Here we refer to the accuracy on the test set without adversarial examples as the \textit{clean accuracy}, and the classifier derived when trained without an adversary as the \textit{base classifier}. We note that for higher $k_{adv}$ one can only expect so much improvement until the $\ell_0$--attack becomes too powerful for any classifier, although we suspect tuning $k$ and running the attack for longer while training can help improve robustness further.\par
To underline our results we refer to the \texttt{Pointwise Attack}, where we display in Table~\ref{tab:median_l0} the values of $\rho$ for our classifiers. We ran $10$ iterations of the attack, utilizing the entire test set of MNIST images. We confirm that $F^{(10)}$ outperforms its unprotected counterpart $F^{(0)}$, and does just as well as the ABS model even when $\beta=1$ \cite{49schott2018adversarially}. Since we know that both the ABS model and $F^{(0)}$ have no robustness guarantees when $\beta=100$, we think it is significant that under this setting $F^{(10)}$ still achieves a high $\rho$ of $17$ pixels.\par
We believe our results for MNIST convey the efficiency and potential of utilizing truncation when designing robust classifiers. We also understand that in order to expand the applicability of truncation, we need to consider how it can be utilized within convolutional neural networks. Unlike with FC layers, the extension of truncation to $2d$-convolutional layers is heuristically motivated, where our approach is directly applying truncation before the first layer of $\textrm{VGG}$--19 \cite{vgg}.\par 
As with FC networks, $\textrm{VGG}^{(0)}$ and its $k$--truncated counterpart $\textrm{VGG}^{(10)}$ were trained with an $\ell_0$--budget $k_{adv}=10$, and attacked with varying time budgets and $\ell_0$--budgets. The results are displayed in Table~\ref{tab:rob_acc}. We see that although $\textrm{VGG}^{(0)}$ is able to maintain a robust accuracy above $0\%$ thanks to adversarial training, we can improve this by adding our truncation component. We also see that the clean accuracy did not suffer when utilizing truncation, and the end result was comparable to the base classifier's accuracy of approximately $91\%$. We think this is significant since prior methods showed large trade-offs between robust accuracy and test set performance \cite{49schott2018adversarially,levine2019robustness}, while truncation combined with adversarial training does strictly better than adversarial training alone.
\newpage
\bibliographystyle{unsrt}
\bibliography{main}

\appendix

\newpage
\section{Proof of Theorem~\ref{thm:lower-bound-diag}}
\label{sec:app-lower-bound}

Here, we propose a strategy for the adversary and use it to prove
Theorem~\ref{thm:lower-bound-diag}.
Recall that $\vnu = \Sigma^{-1/2} \vmu$. Since $\Sigma$ is diagonal, $\nu_i =
\mu_i / \sigma_i$. We will fix a set of
coordinates $A \subseteq [d]$ and a specific value for the budget $k(A) =
\snorm{\vnu_A}_1 \log d$. We introduce a randomized strategy for the adversary  with the following
properties:  (i) it can change up to $k(A)$ coordinates of the input; and (ii)
all the changed coordinates belong to $A$, i.e. the coordinates in $A^c$ are
left untouched. We denote this adversarial strategy by $\adv(A)$.    Given $A \subset [d]$,
having observed $(\vx,y)$, $\adv(A)$ follows the procedure explained below. Let
$\vZ = (Z_1,\cdots, Z_d) \in \reals^d$ be a random vector that $\adv(A)$
constructs using the true input $\vx$. First of all, recall that $\adv(A)$ does
not touch the coordinates that are not in $A$, i.e. for $i \in A^c$ we let $Z_i
= x_i$.  For each $i \in A$, the adversary's act is simple: it either leaves the
value unchanged, i.e. $Z_i = x_i$, or it erases the value, i.e. $Z_i \sim
\text{Unif}[-1,1]$--a completely random value between $-1$ and $+1$. This binary
decision is encoded through a Bernoulli random variable $I_i$ taking value $0$
with probability $p_i(x_i,y)$ and value $1$ otherwise. Here $p_i(x_i,y)$ is
defined  as 
\vspace{-.2cm}
\begin{equation*}
p_i(x_i, y) :=
\begin{cases}
\frac{\exp(-(x_i + y\mu_i)^2 / 2\sigma_i^2)}{\exp(-(x_i - y\mu_i)^2 / 2\sigma_i^2)}  & \text{if  }  \sgn(x_i) = \sgn(y\mu_i) \\
0 & \text{otherwise}
\end{cases}
\end{equation*}
Note that the condition $\sgn(x_i) = \sgn(y\mu_i)$ ensures that $p_i(x_i, y) \leq 1$. In summary, for each $i \in A$, $\adv(A)$ lets 
\begin{equation}
\label{eq:adv-Zi-Ii-def}
Z_i = x_i \times (1-I_i) + \text{Unif}[-1,1] \times I_i,
\end{equation}
where $I_i = \text{Bernoulli} \left(1 - p_i(x_i, y_i) \right)$, and the random variables $I_i$ are generated completely independently w.r.t. all the other variables.  
It is easy to see that the following holds for the conditional density of
$\vZ_A$ given $y$ 
\vspace{-.05cm}
\begin{equation}
\label{eq:adv-Z-cond-f}
f_{\vZ_A|y}(\vec{z}_A | 1) = f_{\vZ_A|y}(\vec{z}_A | -1) = \prod_{i \in A} \left[  \frac{1}{\sqrt{2 \pi \sigma_i^2}} \exp\left( - \frac{(|z_i| + |\mu_i|)^2}{2\sigma_i^2}\right) + \frac{\alpha_i}{2} \one{z_i \in [-1,1]}   \right],
\end{equation}
\vspace{-.15cm}
where for $i \in A$
\begin{equation*}
\alpha_i := \pr{I_i = 1 | y = 1} = \pr{I_i = 1| y = -1} = \int_0^\infty [1-p_i(t,1)] f_{x_i|y}(t|1) dt.
\end{equation*}
In other words, $\alpha_i$ is the probability of changing coordinate $i$. 
Finally, $\adv(A)$ checks if the vectors $\vZ$ and $\vx$ differ within the
budget constraint $k(A) :=  \snorm{\vnu_A}_1 \log d$. 
Define $\vxp$ as follows:
\begin{equation}
\label{eq:vXp-def-A}
\vxp :=
\begin{cases}
\vZ & \text{if } \sum_{i \in A} I_i \leq  \snorm{\vnu_A}_1 \log d\\
\vx & \text{o.t.w.}
\end{cases}
\end{equation}
It can be shown that with
high probability, $\vec{Z}$ is indeed within the specified budget and  $\vxp
= \vZ$.
From this definition, it is evident that with probability one we have
\begin{equation}
\label{eq:vXp-vX-norm-0-log-d-nu-A-1}
\snorm{\vxp - \vx}_0 \leq \snorm{\vnu_A}_1 \log d,
\end{equation}
and hence $\adv(A)$ is a randomized adversarial strategy that only changes the coordinates in $A$ and has budget $k(A) = \snorm{\vnu_A}_1 \log d$. 

Now we use this adversarial strategy to prove
Theorem~\ref{thm:lower-bound-diag}. Before doing so, we need the following
lemma.

\begin{lem}
	\label{lem:error-lower-bound-adv-strategy}
	For any random adversarial strategy with budget $k$ which has a density function $f_{\vxp|\vx, y}$, we have
	\begin{equation*}
	\optloss_{\vmu,\Sigma}(k) \geq \frac{1}{2} \pr{f_{\vxp|y}(\vxp |1) = f_{\vxp|y}(\vxp|-1)} + \pr{f_{\vxp|y}(\vxp|-1) > f_{\vxp|y}(\vxp|1) \bigg| y = 1},
	\end{equation*}
\end{lem}

\begin{proof}
	Note that the right hand side is indeed  the Bayes optimal error associated
	with the MAP estimator assuming that the classifier knows adversary's
	strategy. Since the classifier does not know the adversary's strategy in
	general, the right hand side is indeed a lower bound on the optimal robust
	classification error.
\end{proof}

Now we are ready to prove Theorem~\ref{thm:lower-bound-diag}.

\begin{proof}[Proof of Theorem~\ref{thm:lower-bound-diag}]
	Note that when $A$ is empty, there is no adversarial modification and the
	standard Bayes analysis implies that 
	$\optloss_{\vmu, \Sigma}(0) = \bar{\Phi}(\snorm{\vnu}_2) =
	\bar{\Phi}(\snorm{\vnu_{A^c}}_2)$ and the desired bound holds. Hence, we may
	assume that $A$ is nonempty for the rest of the proof.
	
	Note that due to~\eqref{eq:vXp-vX-norm-0-log-d-nu-A-1}, the 
	randomized strategy $\adv(A)$ is valid for the adversary given the  budget   $\snorm{\vnu}_1 \log d$.
	Thereby  we may use Lemma~\ref{lem:error-lower-bound-adv-strategy} with $\adv(A)$ to bound
	$\optloss_{\vmu, \Sigma}(\snorm{\vnu_A}_1 \log d)$ from below. Before that, we
	show that with high probability under the above randomized strategy for the
	adversary, recalling the definition of random variables $I_i$ for $i \in A$ from
	\eqref{eq:adv-Zi-Ii-def}, we have $\sum_{i \in A} I_i \leq \snorm{\vnu_A}_1 \log d$ and hence
	$\vxp = \vZ$. It is easy to see that for each $i$, $\pr{I_i = 1| y = 1} =
	\pr{I_i = 1 | y = -1}$; therefore,
	\begin{align*}
	\pr{I_i = 1} = \pr{I_i = 1 | y = \sgn(\mu_i)} &= \int_{0}^\infty [1 - p_i(t, \sgn(\mu_i))] f_{x_i | y}(t|\sgn(\mu_i)) d t \\
	&= \int_0^\infty \left[ 1 - \frac{\exp(-(t+|\mu_i|)^2 / 2\sigma_i^2)}{\exp(-(t-|\mu_i|)^2 / 2\sigma_i^2)} \right]\exp\left( -(t-|\mu_i|)^2/2\sigma_i^2 \right) d t \\
	&= 1 - \bar{\Phi}(|\nu_i|) \\
	&= \text{Erf}(|\nu_i|/\sqrt{2}) \\
	&\leq \left( \sqrt{\frac{2}{\pi}} |\nu_i|\right) \wedge 1.
	\end{align*}
	Hence, we have
	\begin{equation*}
	\pr{I_i = 1} = \pr{I_i = 1 | y = 1} = \pr{I_i = 1 | y = -1} \leq \left( \sqrt{\frac{2}{\pi}} |\nu_i|\right) \wedge 1.
	\end{equation*}
	Therefore, using Markov's inequality, if $I$ is the indicator of the event $\sum_{i \in
		A} I_i > \snorm{\vnu_A}_1 \log d$, we have
	\begin{equation}
	\label{eq:pr-I-1-bound-1-logd}
	\pr{I = 1 } = \pr{I= 1 | y=1} = \pr{I = 1 | y = -1}\leq \frac{\sqrt{2 / \pi} \sum_{i \in A} |\nu_i|}{\snorm{\vnu_A}_1 \log d} \leq \frac{1}{\log d}.
	\end{equation}
	Now, we bound
	$\optloss_{\vmu, \Sigma}(\snorm{\vnu_A}_1 \log d)$ from below in 
	the following two cases.

	\underline{Case 1: $A = [d]$}. In this case, using
	Lemma~\ref{lem:error-lower-bound-adv-strategy}, we have
	\begin{align*}
	\optloss_{\vmu, \Sigma}(\snorm{\vnu_A}_1 \log d ) &\geq \frac{1}{2} \pr{f_{\vxp | y}(\vxp | 1) = f_{\vxp|y}(\vxp | -1)} \\
	&\stackrel{(a)}{=} \frac{1}{2} \pr{f_{\vxp | y}(\vxp | 1) = f_{\vxp|y}(\vxp | -1) \,|\, y=1 } \\
	&\geq \frac{1}{2} \pr{f_{\vxp | y}(\vxp | 1) = f_{\vxp|y}(\vxp | -1), I = 0 \,|\, y = 1} \\
	&\stackrel{(b)}{=} \frac{1}{2} \pr{f_{\vZ|y}(\vZ | 1) = f_{\vZ|y}(\vZ|-1) \,|\, y = -1} \\
	&\geq \frac{1}{2} \pr{f_{\vZ|y}(\vZ | 1) = f_{\vZ|y}(\vZ|-1) \,|\, y = 1} - \frac{1}{2} \pr{I = 1\,|\,y = 1}\\
	&\stackrel{(c)}{\geq} \frac{1}{2} - \frac{1}{2 \log d},
	\end{align*}
	where $(a)$ uses the symmetry, $(b)$ uses the fact that when $I = 0$, by
	definition we have $\vxp = \vZ$, and $(c)$ uses~\eqref{eq:adv-Z-cond-f} and~\eqref{eq:pr-I-1-bound-1-logd}.
	
	\underline{Case 2: $A \subsetneqq [d]$}. Using
	Lemma~\ref{lem:error-lower-bound-adv-strategy}, we have
	\begin{equation}
	\label{eq:adv-lowerb-bound-use-lemma-case-2}
	\begin{aligned}
	\optloss_{\vmu, \Sigma}(\snorm{\vnu_A}_1 \log d) &\geq \pr{f_{\vxp| y}(\vxp | -1) > f_{\vxp|y}(\vxp | 1) \,|\, y=1} \\
	&\geq \pr{f_{\vxp| y}(\vxp | -1) > f_{\vxp|y}(\vxp | 1), I = 0 \,|\, y=1} \\
	&\stackrel{(a)}{=}\pr{f_{\vZ| y}(\vZ | -1) > f_{\vZ|y}(\vZ | 1), I = 0 \,|\, y=1} \\
	&\geq \pr{f_{\vZ| y}(\vZ | -1) > f_{\vZ|y}(\vZ | 1)| y=1} - \pr{I = 1\,|\,y = 1} \\
	&\stackrel{(b)}{\geq}  \pr{f_{\vZ| y}(\vZ | -1) > f_{\vZ|y}(\vZ | 1)\,|\, y=1} - \frac{1}{\log d}
	\end{aligned}
	\end{equation}
	where $(a)$ uses the fact that by definition, when $I=0$, we have $\vxp = \vZ$,
	and $(b)$ uses~\eqref{eq:pr-I-1-bound-1-logd}.
	Note that since $Z_i$ are conditionally independent given $y$, we have
	\begin{equation*}
	f_{\vZ|y}(\vZ|y) = f_{\vZ_A | y} (\vZ_A|y) f_{\vZ_{A^c}|y}(\vZ_{A^c}|y).
	\end{equation*}
	But from~\eqref{eq:adv-Z-cond-f}, we have $f_{\vZ_A|y}(\vZ_A|1) =
	f_{\vZ_A|y}(\vZ_A|-1)$ with probability one. Using this
	in~\eqref{eq:adv-lowerb-bound-use-lemma-case-2}, we get
	\begin{equation*}
	\optloss_{\vmu, \Sigma}(\snorm{\vnu_A}_1 \log d) \geq \pr{f_{\vZ_{A^c}|y}(\vZ_{A^c}|-1) > f_{\vZ_{A^c}|y}(\vZ_{A^c}|1) | y = 1} -\frac{1}{\log d} = \bar{\Phi}(\snorm{\vnu_{A^c}}_2) - \frac{1}{\log d}.
	\end{equation*}
	
	We may combine the two cases following the convention that when $A = [d]$, $A^c
	= \emptyset$ and $\snorm{\vnu_{A^c}}_2 = 0$. This completes the proof.
\end{proof}

\section{Proof of Theorem~\ref{thm:alpha-trunc-alpha-opt}}
\label{sec:app-thm-alpha-proof}

Before giving the proof of Theorem~\ref{thm:alpha-trunc-alpha-opt}, we need to
make some definitions and state some lemmas. 
The proofs  of the the  lemmas are provided at the end of this
section. 

We first study the effect of truncation on the inner product.
Lemma~\ref{lem:trun-ip-bound} below from \cite{delgosha2021robust} provides
an upper bound on the deviation of the truncated inner product from the original
inner product. 

\begin{lem}[Lemma 1 in \cite{delgosha2021robust}]
	\label{lem:trun-ip-bound}
	Given $\vx, \vxp, \vw \in \reals^d$, for integer $k$ satisfying $\snorm{\vx - \vxp}_0 \leq k <
	d/2$, we have
	\begin{equation*}
	|\langle  \vw, \vxp \rangle_k - \langle \vw, \vx \rangle | \leq 8k \snorm{\vw \odot \vx}_\infty.
	\end{equation*}
\end{lem}

Recall that in Section~\ref{sec:main-results}, to simplify the discussion, we
restrict ourselves to diagonal covariance matrices. However, in order to have a
general setup, here we  begin by proving an upper bound for the robust classification error of
the family of truncated linear classifiers. In this case, we assume that the
covariance matrix $\Sigma$ is positive definite, but does not need to be
diagonal. Lemma~\ref{lem:general-robust-bound-linear-truncation} below shows an
upper bound for the robust classification error of the $k$--truncated linear
classifier $\mCkw$.

\begin{lem}
	\label{lem:general-robust-bound-linear-truncation}
	We have
	\begin{equation*}
	\loss_{\vmu, \Sigma}(\mCkw, k) \leq \frac{1}{\sqrt{2 \log d}} + \bar{\Phi}\left( \frac{\langle \vw, \vmu\rangle  - 8k \snorm{\tSigma^{1/2} \vw}_\infty (1 + \sqrt{2 \log d})}{\snorm{\Sigma^{1/2} \vw}_2} \right),
	\end{equation*}
	where $\tSigma$ is the diagonal part of $\Sigma$. 
\end{lem}

As a direct consequence, this lemma implies the following bound for the diagonal regime.

\begin{cor}
	\label{cor:linear-trucnated-general-bound-cor-to-diagonal}
	When the covariance matrix $\Sigma$ is diagonal, we have
	\begin{equation*}
	\loss_{\vmu, \Sigma}(\mCkw, k) \leq \frac{1}{\sqrt{2 \log d}} + \bar{\Phi}\left( \frac{\langle \tvw, \vnu\rangle  - 8k \snorm{\tvw}_\infty (1 + \sqrt{2 \log d})}{\snorm{ \tvw}_2} \right),
	\end{equation*}
	where $\tvw = \Sigma^{1/2} \vw$ and $\vnu = \Sigma^{-1/2} \vmu$.
\end{cor}

From this point forward, in order to prove
Theorem~\ref{thm:alpha-trunc-alpha-opt}, we assume that the covariance matrix
$\Sigma$ is diagonal with positive diagonal entries $\sigma_1^2, \dots,
\sigma_d^2$. We define 
\begin{equation}
\label{eq:vnu-def}
\vnu := \Sigma^{-1/2} \vmu,
\end{equation}
so that $\nu_i = \mu_i / \sigma_i$ is the signal to noise ratio associated to
coordinate $i$. Without loss of generality, we may assume that
\begin{equation}
\label{eq:assumption-nu-sorted}
|\nu_1| \geq |\nu_2| \geq \dots \geq |\nu_d|.
\end{equation}
For $\bar{\Phi}(1) < \varepsilon < 1/2$, let $c(\varepsilon)$ be the
unique solution of $\bar{\Phi}(\sqrt{1 - c^2}) = \varepsilon$. Note that
$c(\varepsilon) \in (0,1)$. Moreover, given $c \in (0,1)$, we define
\begin{equation}
\label{eq:lambda-c-def}
\lambda_c := \min\{\lambda: \snorm{\vnu_{[1:\lambda]}}_2 \geq c\}.
\end{equation}
Note that since $c > 0$, we have $\lambda_c \geq 1$. Moreover, since $c < 1$ and
$\snorm{\vnu}_2 = 1$, we have
\begin{equation}
\label{eq:lambda_c_less_d}
\lambda_c < d.
\end{equation}

Using Lemma~\ref{lem:general-robust-bound-linear-truncation} and in particular
Corollary~\ref{cor:linear-trucnated-general-bound-cor-to-diagonal} in the
diagonal regime, we can show the following bound on the robust classification
error of the optimal $k$--truncated linear classifier $\mC^{(k)}_{\vw^*(k)}$.

\begin{lem}
	\label{lem:loss-k-a-norm-1-bound}
	Assume that the covariance matrix $\Sigma$ is diagonal. 
	Given $\bar{\Phi}(1) < \varepsilon < 1/2$, for $k = a \snorm{\vnu_{[1:\lambdace]}
	}_1$, we have
	\begin{equation*}
	\loss(\mC^{(k)}_{\vw^*(k)}, k) \leq \varepsilon + a \frac{8(1+\sqrt{2 \log d})}{\sqrt{2 \pi} \sqrt{1 - c(\varepsilon)^2}} + \frac{1}{\sqrt{2 \log d}}.
	\end{equation*}
\end{lem}

Furthermore, we can show the following lower bound on $\kopt$ which involves
the class of all classifiers. 

\begin{lem}
	\label{lem:kopt-bound}
	For $\bar{\Phi}(1) + \frac{1}{\log d} < \varepsilon < \frac{1}{2}$, we have
	\begin{equation*}
	\kopt\left(\varepsilon - \frac{1}{ \log d} \right) \leq \snorm{\vnu_{[1:\lambda_{c(\varepsilon)}]}}_1 \log d.
	\end{equation*}
\end{lem}

We are finally ready to prove Theorem~\ref{thm:alpha-trunc-alpha-opt}.

\begin{proof}[Proof of Theorem~\ref{thm:alpha-trunc-alpha-opt}]
	Using Lemma~\ref{lem:loss-k-a-norm-1-bound} for $\bar{\Phi}(1) \leq
	\varepsilon < 1/2$ and $k = a
	\snorm{\vnu_{[1:\lambdace]}}$ with
	\begin{equation*}
	a = \sqrt{1 - c(\varepsilon)^2} \frac{1}{16 \log d},
	\end{equation*}
	we get
	\begin{equation*}
	\loss(\mC^{(k)}_{\vw^*(k)}, k) \leq \varepsilon + \frac{1 + \sqrt{2 \log d}}{2 \sqrt{2 \pi} \log d} + \frac{1}{\sqrt{2 \log d}} \leq \varepsilon + \sqrt{\frac{2}{\log d}}.
	\end{equation*}
	This means that
	\begin{equation}
	\label{eq:ktrunc-lower-bound-1}
	\ktrunc\left( \varepsilon + \sqrt{\frac{2}{\log d}} \right) \geq \frac{\sqrt{1 - c(\varepsilon)^2}}{16} \frac{\snorm{\vnu_{[1:\lambdace]}}_1}{\log d} \qquad \text{ for } \quad \bar{\Phi}(1) < \varepsilon < \frac{1}{2} - \sqrt{\frac{2}{\log d}}.
	\end{equation}
	On the other hand, from Lemma~\ref{lem:kopt-bound} we know that
	\begin{equation*}
	\kopt\left(\varepsilon - \frac{1}{ \log d} \right) \leq \snorm{\vnu_{[1:\lambda_{c(\varepsilon)}]}}_1 \log d \qquad \text{ for } \quad \bar{\Phi}(1) + \frac{1}{\log d} < \varepsilon < \frac{1}{2}.
	\end{equation*}
	Comparing this with~\eqref{eq:optloss-vnu-lambda-c-bound-1}, we realize that
	\begin{equation*}
	\ktrunc\left( \varepsilon + \sqrt{\frac{2}{\log d}} \right) \geq \frac{\sqrt{1 - c(\varepsilon)^2}}{16 \log^2 d} \kopt\left( \varepsilon - \frac{1}{\log d} \right) \qquad \text{ for } \quad \bar{\Phi}(1) + \frac{1}{\log d} < \varepsilon < \frac{1}{2} - \sqrt{\frac{2}{\log d}}.
	\end{equation*}
	Equivalently, taking $\log_d$ from both sides, we realize that for
	$\bar{\Phi}(1) + 1/ \log d < \varepsilon < 1/2 - \sqrt{2 / \log d}$, 
	\begin{align*}
	\alphatrunc\left( \varepsilon + \sqrt{\frac{2}{\log d}} \right) \geq \alphaopt\left( \varepsilon - \frac{1}{\log d} \right) - \frac{2 \log \log d}{\log d} + \frac{\log(\sqrt{1 - c(\varepsilon)^2} / 16)}{\log d}.
	\end{align*}
	By shifting $\varepsilon$,
	we
	realize that for $\bar{\Phi}(1) + 1/ \log d + \sqrt{2 / \log d} < \varepsilon
	< 1/2$, we have
	\begin{equation}
	\label{eq:alphatrunc-alphaopt-c1-c2-proof}
	\alphatrunc(\varepsilon) \geq \alphaopt(\varepsilon - c_1(\varepsilon, d)) - c_2(\varepsilon, d),
	\end{equation}
	where
	\begin{equation}
	\label{eq:c1-e-d-def}
	c_1(\varepsilon, d) := \frac{1}{\log d} + \sqrt{\frac{2}{\log d}},
	\end{equation}
	and
	\begin{equation}
	\label{eq:c2-e-d-def}
	c_2(\varepsilon, d) := \frac{2 \log \log d}{\log d} - \frac{\log \left( \frac{\sqrt{1-c(\varepsilon - \sqrt{2 / \log d})^2}}{16} \right)}{\log d}.
	\end{equation}
	Observe that for $i \in \{1,2\}$, $c_i(\varepsilon, d)$ does not depend on the
	parameters of the problem and $\lim_{d \rightarrow \infty} c_i(\varepsilon, d)
	= 0$. On the other hand, since  $\alphaopt(\varepsilon)$ is obtained by
	optimizing over all classifiers while $\alphatrunc(\varepsilon)$ is obtained
	by optimizing over the class of linear truncated classifiers, we always have
	$\alphaopt(\varepsilon) \geq \alphatrunc(\varepsilon)$. This completes the proof.
\end{proof}

Finally, we give the proofs for
Lemmas~\ref{lem:general-robust-bound-linear-truncation},~\ref{lem:loss-k-a-norm-1-bound},
and \ref{lem:kopt-bound}.

\begin{proof}[Proof of Lemma~\ref{lem:general-robust-bound-linear-truncation}]
	We may write
	\begin{equation}
	\label{eq:robust-error-general-w-pr-exists-vxp}
	\begin{aligned}
	\loss_{\vmu, \Sigma}(\mCkw, k) &= \evwrt{(\vx, y) \sim \mD}{\max_{\vxp \in \mB_0(\vx, k)} \ell(\mCkw; \vxp, y)} \\
	&= \pr{\exists \vxp \in \mB_0(\vx, k): \mCkw(\vxp \neq y)} \\
	&= \pr{\exists \vxp \in \mB_0(\vx, k): \sgn(\langle \vw, \vxp \rangle_k) \neq y} \\
	&\stackrel{(*)}{=} \pr{\exists \vxp \in \mB_0(\vx, k): \sgn(\langle \vw, \vxp \rangle_k) \neq 1 | y = 1} \\
	&= \pr{\exists \vxp \in \mB_0(\vx, k): \langle \vw, \vxp \rangle_k \leq 0 | y = 1}
	\end{aligned}
	\end{equation}
	where $(*)$ uses the symmetry in distribution $\mD$. Using Lemma~\ref{lem:trun-ip-bound}, for all $\vxp \in \mB_0(\vx, k)$, we have
	\begin{equation*}
	\langle \vw, \vxp \rangle_k \geq \langle \vw, x \rangle - 8k \snorm{\vw \odot \vx}_\infty.
	\end{equation*}
	Using this in~\eqref{eq:robust-error-general-w-pr-exists-vxp}, we get
	\begin{equation}
	\label{eq:robust-loss-simplified-norm-inf}
	\loss_{\vmu, \Sigma}(\mCkw, k) = \pr{\langle  \vw, \vx \rangle \leq 8 k \snorm{\vw \odot \vx}_\infty | y = 1}.
	\end{equation}
	Note that conditioned on $y = 1$, we have $\vx = \vmu + \vz$ where $\vz \sim
	\mN(0,\Sigma)$. Let $\tSigma$ be the diagonal matrix consisting of the diagonal
	entries in $\Sigma$. Since $\tSigma$ is diagonal, we may write
	\begin{equation}
	\label{eq:w-odot-x-bound}
	\begin{aligned}
	\snorm{\vw \odot \vx}_\infty &= \snorm{\vw \odot \vmu + \vw \odot \vz}_\infty \\
	&\leq \snorm{\vw \odot \vmu}_\infty + \snorm{\vw \odot \vz}_\infty \\
	&\leq \snorm{(\tSigma^{1/2}\vw) \odot (\tSigma^{-1/2}\vmu)}_\infty + \snorm{(\tSigma^{1/2}\vw) \odot (\tSigma^{-1/2}\vz)}_\infty \\
	&\leq \snorm{\tSigma^{1/2} \vw}_\infty\left( \snorm{\tSigma^{-1/2} \vmu}_\infty + \snorm{\tSigma^{-1/2}\vz}_\infty \right).
	\end{aligned}
	\end{equation}
	We now bound the infinity norm of the vector  $\va := \tSigma^{-1/2}
	\vmu$. With $\sigma_1^2, \dots, \sigma_d^2$ denoting the diagonal entries in
	$\Sigma$, we have $a_i = \mu_i / \sigma_i$. Note that $\bar{\Phi}(|\mu_i| /
	\sigma_i)$  is the optimal Bayes classification error of $y$ given  $x_i$
	only, which cannot be smaller than the optimal Bayes classifier of $y$ given
	the whole vector $\vx$, which is in turn equal to
	$\bar{\Phi}(\snorm{\Sigma^{-1/2} \vmu}_2) = \bar{\Phi}(1)$. This means that
	$|a_i| = |\mu_i| / \sigma_i \leq 1$, and in particular
	\begin{equation}
	\label{eq:tsigma-1/2-vmu-infty-less-1}
	\snorm{\va}_\infty = \snorm{\tSigma^{-1/2} \vmu}_\infty \leq 1.
	\end{equation}
	Next, we bound the infinity norm of the random vector $\vb :=
	\tSigma^{-1/2}\vz$. Note that $b_i \sim \mN(0,1)$. Therefore, using the union
	bound, we may write
	\begin{equation}
	\label{eq:norm-inf-tsigma-1/2-z-sqrt-2-logd}
	\begin{aligned}
	\pr{\snorm{\tSigma^{-1/2}\vz}_\infty \geq \sqrt{2 \log d}} &\leq d \bar{\Phi}(\sqrt{2 \log d}) \\
	&\leq d \frac{1}{\sqrt{2 \pi} \sqrt{2 \log d}} e^{- \log d} \\
	&\leq \frac{1}{\sqrt{2 \log d}}.
	\end{aligned}
	\end{equation}
	Using this together with~\eqref{eq:tsigma-1/2-vmu-infty-less-1} back
	into~\eqref{eq:w-odot-x-bound}, we realize that
	\begin{equation*}
	\pr{\snorm{\vw \odot \vx}_\infty \leq \snorm{\tSigma^{1/2} \vw}_\infty (1 + \sqrt{2 \log d})} \geq 1 - \frac{1}{\sqrt{2 \log d}}.
	\end{equation*}
	This together with~\eqref{eq:robust-loss-simplified-norm-inf} implies that
	\begin{equation}
	\label{eq:loss-general-w-bound-semi-final}
	\loss_{\vmu, \Sigma}(\mCkw, k) \leq \frac{1}{\sqrt{2 \log d}} + \pr{\langle \vw , \vx \rangle \leq 8k \snorm{\tSigma^{1/2} \vw}_\infty ( 1 + \sqrt{2 \log d})\, \Big |\, y = 1}.
	\end{equation}
	Again, using the fact that $\vx = \vmu + \vz$ conditioned on $y = 1$, we have
	\begin{align*}
	&\pr{\langle \vw , \vx \rangle \leq 8k \snorm{\tSigma^{1/2} \vw}_\infty ( 1 + \sqrt{2 \log d})\, \Big |\, y = 1} \\
	&\qquad= \pr{\langle \vw, \vz \rangle \leq 8k \snorm{\tSigma^{1/2} \vw}_\infty ( 1 + \sqrt{2 \log d}) - \langle {\vw, \vmu} \rangle} \\
	&\qquad= \pr{\frac{\langle \vw, \vz \rangle}{\snorm{\Sigma^{1/2} \vw}_2} \leq \frac{8k \snorm{\tSigma^{1/2} \vw}_\infty ( 1 + \sqrt{2 \log d}) - \langle {\vw, \vmu\rangle}}{\snorm{\Sigma^{1/2} \vw}_2} }  \\
	&\qquad=\bar{\Phi}\left( \frac{\langle \vw, \vmu\rangle  - 8k \snorm{\tSigma^{1/2} \vw}_\infty (1 + \sqrt{2 \log d})}{\snorm{\Sigma^{1/2} \vw}_2} \right).
	\end{align*}
	Substituting this into~\eqref{eq:loss-general-w-bound-semi-final} completes
	the proof of Lemma~\ref{lem:general-robust-bound-linear-truncation}.
\end{proof}

\begin{proof}[Proof of Lemma~\ref{lem:loss-k-a-norm-1-bound}]
	We define $\tvw \in \reals^d$  as follows
	\begin{equation*}
	\tw_i =
	\begin{cases}
	0 & i < \lambdace \\
	\nu_i & i \geq \lambdace
	\end{cases}
	\end{equation*}
	With this, let $\vw = \Sigma^{-1/2} \tvw$ and note that since $\vw^*(k)$ is
	obtained by optimizing for $\loss(\mC^{(k)}_{\vw}, k)$, we have
	\begin{equation}
	\label{eq:mckw-star-less-than-mckw-candidate}
	\loss(\mC^{(k)}_{\vw^*(k)}, k) \leq \loss(\mC^{(k)}_{\vw}, k), 
	\end{equation}
	with $\vw$ defined above. From
	Corollary~\ref{cor:linear-trucnated-general-bound-cor-to-diagonal}, we have
	\begin{equation}
	\label{eq:loss-mckw-candidate-bound-from-cor}
	\loss(\mC^{(k)}_{\vw}, k) \leq \frac{1}{\sqrt{2 \log d}} + \bar{\Phi}\left(  \frac{\langle \tvw, \vnu\rangle  - 8k \snorm{\tvw}_\infty (1 + \sqrt{2 \log d})}{\snorm{ \tvw}_2} \right).
	\end{equation}
	Note that
	\begin{equation}
	\label{eq:ip-tvw-tnu-lambdac-d-norm2}
	\langle  \tvw, \vnu \rangle = \sum_{i=\lambdace}^d \nu_i^2 = \snorm{\vnu_{[\lambdace:d]}}_2^2.
	\end{equation}
	Likewise,
	\begin{equation}
	\label{eq:norm-tvw-nu-2}
	\snorm{\tvw}_2 = \sqrt{\sum_{i=\lambdace}^d \nu_i^2} = \snorm{\vnu_{[\lambdace:d]}}_2.
	\end{equation}
	Recall that $\lambdace$ by definition is the smallest $\lambda$ such that
	$\snorm{\vnu_{[1:\lambda]}}_2 \geq c(\varepsilon)$. This implies that
	$\snorm{\vnu_{[1:\lambdace-1]}}_2 < c(\varepsilon)$ and
	\begin{equation}
	\label{eq:norm-cnu-lambdace-d-more-1-c2}
	\snorm{\vnu_{[\lambdace:d]}}_2^2 = \snorm{\vnu}_2^2 - \snorm{\vnu_{[1:\lambdace-1]}}_2^2 \geq 1 - c(\varepsilon)^2.
	\end{equation}
	Comparing this with~\eqref{eq:ip-tvw-tnu-lambdac-d-norm2} and
	\eqref{eq:norm-cnu-lambdace-d-more-1-c2}, we realize that
	\begin{equation}
	\label{eq:frac-ip-vw-nu-norm-tvw-2-sqrt-1-c2}
	\frac{\langle \tvw, \vnu \rangle}{\snorm{\tvw}_2} = \snorm{\vnu_{[\lambdace:d]}}_2 \geq \sqrt{1 - c(\varepsilon)^2}.
	\end{equation}
	On the other hand, since we have assumed in~\eqref{eq:assumption-nu-sorted}, we
	have $\snorm{\tvw}_\infty = |\vnu_{\lambdace}|$. Furthermore,
	using~\eqref{eq:assumption-nu-sorted}, we have
	\begin{equation*}
	\snorm{\vnu_{[1:\lambdace]}}_1 |\nu_{\lambdace}| = \sum_{i=1}^{\lambdace} |\nu_i| |\nu_{\lambdace}| \leq \sum_{i=1}^{\lambdace} |\nu_i|^2 \leq \snorm{\vnu}_2^2 = 1.
	\end{equation*}
	This together with~\eqref{eq:ip-tvw-tnu-lambdac-d-norm2}
	and~\eqref{eq:norm-tvw-nu-2} implies that 
	\begin{align*}
	\frac{8k \snorm{\tvw}_\infty (1 + \sqrt{2 \log d})}{\snorm{\vw}_2} &= \frac{8(1 + \sqrt{2 \log d}) a \snorm{\vnu_{[1:\lambdace]}}_1 |\nu_{\lambdace}|}{\snorm{\vnu_{[\lambdace:d]}}_2} \\
	&\leq \frac{8a (1 + \sqrt{2 \log d})}{\sqrt{1 - c(\varepsilon)^2}}.
	\end{align*}
	Using this and~\eqref{eq:frac-ip-vw-nu-norm-tvw-2-sqrt-1-c2} back
	into~\eqref{eq:loss-mckw-candidate-bound-from-cor} and using the fact that
	$\bar{\Phi}(.)$ is decreasing and $1/\sqrt{2 \pi}$-Lipschitz, we realize that
	\begin{align*}
	\loss(\mCkw, k) &\leq \frac{1}{\sqrt{2 \log d}} + \bar{\Phi}\left( \sqrt{1 - c(\varepsilon)^2} - \frac{8a (1 + \sqrt{2 \log d})}{\sqrt{1 - c(\varepsilon)^2}} \right) \\
	&\leq \bar{\Phi}(\sqrt{1 - c(\varepsilon)^2})  + a \frac{8 (1 + \sqrt{2 \log d})}{\sqrt{2 \pi}\sqrt{1 - c(\varepsilon)^2}} + \frac{1}{\sqrt{2 \log d}}\\
	&= \varepsilon  +a \frac{8 (1 + \sqrt{2 \log d})}{\sqrt{2 \pi}\sqrt{1 - c(\varepsilon)^2}} + \frac{1}{\sqrt{2 \log d}}.
	\end{align*}
	This together with~\eqref{eq:mckw-star-less-than-mckw-candidate} completes the proof.
\end{proof}

\begin{proof}[Proof of Lemma~\ref{lem:kopt-bound}]
	From Theorem~\ref{thm:lower-bound-diag}, we have
	\begin{equation}
	\label{eq:optloss-vnu-lambda-c-bound-1}
	\optloss(\snorm{\vnu_{[1:\lambda_{c(\varepsilon)}]}}_1 \log d) \geq \bar{\Phi}(\snorm{\vnu_{[1+\lambda_{c(\varepsilon)}: d]}}_2) - \frac{1}{\log d}.
	\end{equation}
	We have
	\begin{align*}
	\snorm{\vnu_{[1+\lambda_{c(\varepsilon)}: d]}}_2^2 &= \snorm{\vnu}_2^2 - \snorm{\vnu_{[1:\lambda_{c(\varepsilon)}]}}_2^2 \\
	&= 1 - \snorm{\vnu_{[1:\lambda_{c(\varepsilon)}]}}_2^2 \\
	&\leq 1 -c(\varepsilon)^2.
	\end{align*}
	where the last inequality uses the definition of $\lambda_{c(\varepsilon)}$
	in~\eqref{eq:lambda-c-def}. Using this back
	into~\eqref{eq:optloss-vnu-lambda-c-bound-1}, we get
	\begin{equation*}
	\optloss(\snorm{\vnu_{[1:\lambda_{c(\varepsilon)}]}}_1 \log d) \geq \bar{\Phi}(\sqrt{1 - c(\varepsilon)^2}) - \frac{1}{\log d} = \varepsilon - \frac{1}{\log d}.
	\end{equation*}
	Note $\optloss(k)$ is nondecreasing in $k$, therefore this implies that
	$\optloss(k) \geq \varepsilon - 1 / \log d$ for $k \geq
	\snorm{\vnu_{[1:\lambda_{c(\varepsilon)}]}}_1 \log d$ and completes the proof.
\end{proof}

\section{Implementation Details}\label{sec:app_implementation}
\subsection{Architecture and training details for MNIST}
For our experiments on MNIST, we utilized fully connected networks consisting of 5 hidden layers as shown below in Table~\ref{tab:FC_arch}. For the truncated version, we replaced the first FC layer with our matrix truncation operation defined in \eqref{eq:matrix_trunc}. The exact implementation and code required to replicate our results are given as part of the attached supplementary material.\par
\begin{table}[!h]
	\captionsetup{skip=\baselineskip}
	\centering
	\begin{tabular}{ll}
		\toprule
		\multicolumn{1}{c}{Layer} &
		\multicolumn{1}{c}{Output Shape}\\
		\midrule
		Input & $784 (28\times28)$\\
		Fully Connected + ReLU & $1568$\\
		Fully Connected + ReLU & $3136$\\
		Fully Connected + ReLU & $500$\\
		Fully Connected + ReLU & $100$\\
		Fully Connected & $10$\\
		\bottomrule
	\end{tabular}
	\caption{Fully Connected Network Architecture}
	\label{tab:FC_arch}
	\caption*{Architecture of $F^{(0)}$, where for $F^{(k)}$ the first FC Layer gets replaced with the matrix truncation operator defined in \eqref{eq:matrix_trunc}.} 
\end{table}
For training $F^{(10)}$ and $F^{(0)}$, we utilized stochastic gradient descent and reset the training set with adversarial examples every $25$ epochs using the \texttt{sparse-rs} attack with an $\ell_0$--budget of $10$ and a time budget of $300$ queries. The rest of the details for the learning component are provided in Table~\ref{tab:FC_params} below. As previously mentioned, the full implementation along with a general adversarial training class is provided as part of our code. 
\begin{table}[!h]
	\captionsetup{skip=\baselineskip}
	\centering
	\begin{tabular}{ll}
		\toprule
		\multicolumn{1}{c}{Parameter} &
		\multicolumn{1}{c}{Description}\\
		\midrule
		Batch Size & $256$ \\
		Optimizer & Stochastic Gradient Descent\\
		Training Epochs & $250$\\
		Learning Rate & $0.001$\\
		Momentum & $0.9$\\
		\bottomrule
	\end{tabular}
	\caption{Training details for MNIST}
	\label{tab:FC_params}
	\caption*{Details of the stochastic gradient descent implementation used for training $F^{(10)}$ and $F^{(0)}$.}
\end{table}
\subsection{Architecture and training details for CIFAR}
As done by previous works referred to in the main text, we used the CIFAR-10 dataset specifically when performing our experiments. For our network structure we chose the VGG-19 \cite{vgg} architecture, implementing it without dropout layers. We do not show the architecture here due to its size, but the full implementation is provided in our code. For the truncated version $\textrm{VGG}^{(10)}$, we applied truncation as defined in \eqref{eq: truncation} before the first convolution layer. As with FC networks, when training $\textrm{VGG}^{(10)}$ and $\textrm{VGG}^{(0)}$ we utilized stochastic gradient descent, resetting the training set with adversarial examples every $25$ epochs using the \texttt{sparse-rs} attack with an $\ell_0$--budget of $10$ and a time budget of $300$ queries. The rest of the training details are provided in Table~\ref{tab:Conv_params} below.
\begin{table}[!h]
	\captionsetup{skip=\baselineskip}
	\caption{Training details for CIFAR}
	\label{tab:Conv_params}
	\centering
	\begin{tabular}{ll}
		\toprule
		\multicolumn{1}{c}{Parameter} &
		\multicolumn{1}{c}{Description}\\
		\midrule
		Batch Size & $128$ \\
		Optimizer & Stochastic Gradient Descent\\
		Training Epochs & $250$\\
		Learning Rate & $0.2,0.1,0.05,0.025,0.01,0.005,0.0025,0.001,0.0005,0.00025$\\
		Momentum & $0.9$\\
		Weight Decay & $0.0005$\\
		\bottomrule
	\end{tabular}
	\caption*{Details of the stochastic gradient descent implementation used for training $\textrm{VGG}^{(10)}$ and $\textrm{VGG}^{(0)}$. Note that the learning rate was updated in descending order every $25$ epochs according to the list provided.}
\end{table}
\subsection{Efficiency of implementation}
As noted, the exact implementations of truncation both within FC and $\textrm{VGG}$ networks is provided as part of our submitted code. Here we would like to point out some details regarding the efficiency of our implementations. For FC networks, we built a custom \texttt{pytorch} module to implement \eqref{eq:matrix_trunc} as a linear layer that performs truncation at every vector dot product before returning the output. This meant we could no longer rely on \texttt{pytorch}'s computationally efficient batch matrix multiplication operation that is written in \texttt{c++}, and instead broadcasted our operation to work on batches using \texttt{python}. Utilizing the efficient FC layers to train $F^{(0)}$, each training epoch took roughly $0.5$ seconds on an RTX-3080 GPU, while for the truncated network $F^{(10)}$ each training epoch took $29-30$ seconds. Decreasing the truncation parameter to $k=1$ we see the truncated network $F^{(1)}$ takes $26-27$ seconds per epoch, showing that truncation is not the sole reason behind the slow down, and this can be mitigated by implementing our custom layer in \texttt{c++} as done for the regular FC layer. 

We actually see this fact come to play for our implementation of the $\textrm{VGG}$ networks, as we did not utilize a custom \texttt{pytorch} module, and instead implemented truncation as a separate function. The regular $\textrm{VGG}^{(0)}$ network took $13-14$ seconds per epoch while the truncated $\textrm{VGG}^{(10)}$ network took $32-33$ seconds. This shows that truncation can be implemented efficiently, and does not significantly increase the computational overhead for deep neural networks. 
\subsection{Computational resources}
The majority of our work was performed on an internal cluster containing the 20C/40T Intel Xeon Silver 4114 CPU, 64GB RAM, and $2\times$GTX-1080 GPUs. All adversarial training was performed on the GPUs, where fully training $F{(10)}$ and $\textrm{VGG}^{(10)}$ as done in Table~\ref{tab:rob_acc} took roughly 5 days each. For analyzing the attacks and running shorter experiments, a  personal computer with the 8C/16T Intel-9900K CPU, 32GB RAM, and an RTX-3080 GPU was used. The efficiency of our implementations was compared using the personal computer, as was described in the previous section. \par
\end{document}